\documentclass{article}

\usepackage[preprint,nonatbib]{nips_2018}

\usepackage{graphicx}
\graphicspath{{./images/}}
\DeclareGraphicsExtensions{.pdf,.png}

\usepackage[utf8]{inputenc} 
\usepackage[T1]{fontenc}    
\usepackage{hyperref}       
\usepackage{url}            
\usepackage{booktabs}       
\usepackage{amsfonts}       
\usepackage{nicefrac}       
\usepackage{microtype}      
\usepackage{amsmath}
\usepackage{xcolor}
\usepackage{amsthm} 
\usepackage{amssymb} 
\usepackage{thmtools, thm-restate}

\title{Imagining the Unseen: Learning a Distribution over Incomplete Images with Dense Latent Trees}

\author{
    Sebastian Kaltwang\\
    Five AI Ltd.\\
    Cambridge, U.K.\\
    \texttt{sebastian@five.ai} \\
    \And
    Sina Samangooei \\
    Five AI Ltd.\\
    Cambridge, U.K.\\
    \texttt{sina@five.ai} \\
    \And
    John Redford \\
    Five AI Ltd.\\
    Cambridge, U.K.\\
    \texttt{john@five.ai} \\
    \And
    Andrew Blake \\
    Five AI Ltd.\\
    Cambridge, U.K.\\
    \texttt{andrew@five.ai} \\
}

\begin{document}


\maketitle

\begin{abstract}
Images are composed as a hierarchy of object parts.
We use this insight to create a generative graphical model that defines a hierarchical distribution over image parts.
Typically, this leads to intractable inference due to loops in the graph.
We propose an alternative model structure, the Dense Latent Tree (DLT), which avoids loops and allows for efficient exact inference, while maintaining a dense connectivity between parts of the hierarchy.
The usefulness of DLTs is shown for the example task
of image completion on partially observed MNIST \cite{lecun1998mnist} and Fashion-MNIST \cite{xiao2017fashion} data.
We verify having successfully learned a hierarchical model of images by visualising its latent states.
\end{abstract}

\section{Introduction}


Hierarchical structures are abundant in natural images. Scenes are composed of objects, which are composed of parts, which are composed of smaller parts and so on.
In this work we propose a generative graphical model over image pixels that has this hierarchical composition at its core, called Dense Latent Trees (DLTs).
Take for example the MNIST digits \cite{lecun1998mnist}: each digit consists of parts (curves and lines), which consist of edges (at different angles), which consist of pixels (black or white).
The larger parts are composed from a spatially arranged set of smaller parts. Furthermore, parts are often self-similar, and thus can be shared between the larger parts.
E.g. the digits 6 and 9 can share the same `circle' component, but placed at a different spatial position. Also, each of the circles can share the same type of edges.
Since a part can potentially occur at any position in the image, we account for all overlapping positions with a `dense' distribution of parts.
Typically, this dense connectivity implies a non-tree shaped graphical model with intractable exact inference (e.g. \cite{Salakhutdinov2009Deep}).
In Sec.~\ref{sec:model}, we propose an alternative DLT structure with efficient inference and sampling algorithms.
The main idea is to use a densely connected (non-tree) model as template and then break up all loops via duplicating child nodes with multiple parents.
Each parent gets its own separate child, which results in tree structured graph. See Sec.~\ref{sec:structure} for details.

The utility of such a generative image model is many-fold, and includes classification, clustering, in-painting and super-resolution.
We pick one example task and show in Sec.~\ref{sec:experiments} that our model is able to complete missing image parts in a plausible way.
As a special case, we are able to \emph{train} our model with missing image parts, i.e. we can handle problems where complete images are not available.
Finally, we visualise the learned parts to verify that the model has successfully captured a hierarchical distribution over image parts.

Our contributions include:
(1) creating a graphical model that defines images as a distribution of overlapping image parts,
(2) learning of the image parts without supervision from incomplete data,
(3) formulating efficient inference, learning and sampling in linear time despite quadratically many random variables, and
(4) showing a detailed analysis of the learned part distributions and the efficacy of the proposed structure for image completion on two datasets with randomly missing parts.

\section{Related Work}

%

Previous work on image models include Quad-trees (e.g. \cite{bouman1994multiscale}), which also feature a hierarchical decomposition of images.
Input pixels are split into \emph{non-overlapping} patches of size 2x2 and each hidden node with its children can be seen as a distribution over object parts.
The model is able to express a strong joint distribution for pixels within the same image patch.
However, neighbouring pixels in different patches have a weak correlation, depending on the distance to their lowest common ancestor.
This leads to block artefacts with discretised borders.
In contrast to that, DLTs partition the image into \emph{overlapping} patches, and thus model a strong connection between all neighbouring pixels pairs.
Slorkey and Williams \cite{slorkey2003image} overcome the fixed quad-tree structure by treating node connectivity as a random variable, which is estimated during inference.
This allows for dynamically placing the image parts within the upper layers, but overlap between parts is not possible.

%
Another strand of research has used the Restricted Boltzmann Machine (RBM) to create a distribution over images with dense connectivity between pixels \cite{Salakhutdinov2009Deep,eslami2014shape}.
Each of the hidden variables can be seen as modelling one factor in the input image parts. The extension  \cite{eslami2012generative} models explicitly overlapping parts within a hierarchy. 
While able to natively handle missing inputs, the main drawback of RBMs its intractable inference which results in slow approximate sampling and learning algorithms.

Topic models have been proposed for images \cite{sudderth2005learning}, where objects are decomposed of latent parts. In contrast to DLT, there is only a single latent layer for the parts, i.e. there is no hierarchy for decomposing larger parts into multiple smaller parts.


%
Various recent work has utilised CNNs for creating a generative distribution over images: Variational Auto-encoders (VAEs) \cite{kingma2013auto}, Generative Adversarial Networks (GANs) \cite{goodfellow2014generative} and the PixelCNN \cite{van2016conditional}.
While excelling at sampling images from a prior or a fixed conditional input, all of the methods require complete data during training.
The encoder (in VAEs and GANs) and conditioning CNN (PixelCNN) require all inputs being present.
Moreover, these methods do not model a distribution of parts: VAEs and GANs learn a deterministic map from a multi-dimensional random variable to the image space.
PixelCNNs model a distribution of pixels, conditioned on the previously generated image parts.

%
%
%

\section{Dense Latent Trees}\label{sec:model}

A dense latent tree (DLT) is a parametric probabilistic graphical model, that defines a joint probability distribution $p$ over a set of
discrete random variables $\mathbf{X}=(\mathbf{O},\mathbf{H})$. The subset $\mathbf{O}$ is observed and the remaining set $\mathbf{H}$ is hidden or latent.
 The likelihood for $\mathbf{O}$ is given by:
\begin{equation}
p(\mathbf{O})=\sum_{\mathbf{H}}p(\mathbf{O},\mathbf{H})
\end{equation}
where the sum extends over all states of all variables in $\mathbf{H}$.
The model is constructed by combining multiple applications of the core structure called ``kernel''%
\footnote{We denote the core structure of DLT as ``kernel'' in analogy to CNNs, since it is repeatedly applied with shared parameters at different spatial locations.}
(see Sec.~\ref{sec:kernel}) to form the overall DLT structure (see Sec.~\ref{sec:structure}). 
Then we explain how to obtain the likelihood (in Sec.~\ref{sec:likelihood}), learn the parameters (in Sec.~\ref{sec:learning}) and sample from the DLT (in Sec.~\ref{sec:sampling}).

\subsection{Kernel}\label{sec:kernel}

The kernel is defined by a probabilistic graphical model, where the nodes are random variables and the edges are parametric conditional probability distributions.
For simplicity, we chose categorical distributions for all nodes.
However, any choice with tractable marginal inference would be feasible. The kernel has $K$ input nodes
$\mathbf{x}_{k}$, $k \in \{0,..., K-1\}$, each of them representing one pixel or an image part.
The inputs are connected to a latent node $\mathbf{y}$ as parent, which also has a parametric prior.
This parent represents the ``part of parts'' in the image composition. 
Fig.~\ref{fig:Structure_kernel} shows an example for $K=3$. We denote the edge parameters also as ``weights'' of the kernel.
\begin{figure}[htbp]
    \centering
    \includegraphics[scale=0.45]{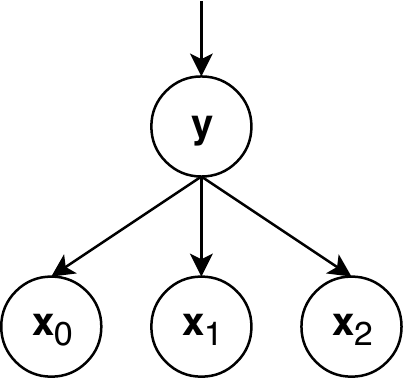}
    \caption{Example of a DLT kernel, which is repeatedly applied with shared weights at each layer.}
    \label{fig:Structure_kernel}
\end{figure}
The model describes a mixture of Categorical variables and is also known as Latent Class Analysis \cite{lazarsfeld1968latent}.
It is fully tractable and we can easily calculate the likelihood, sample, train the parameters and infer a posterior distribution given an observed subset of the nodes using belief propagation (BP) \cite{pearl1988probabilistic}.

\subsection{Structure}\label{sec:structure}
The nodes of a DLT are organized in layers. 
The input image is represented as a set of observed nodes $\mathbf{O}$ in the lowest layer. 
$\mathbf{O}$ is spatially arranged in 2-D, where each node corresponds to a pixel. 
The learned parts are represented as hidden nodes $\mathbf{H}$, which are spread over the upper layers. In case of partially observed inputs, the unobserved pixel in the lowest layer become hidden and thus a part of $\mathbf{H}$. 
To improve readability, we explain the DLT structure on the example of the simplest case, i.e. 1-D images as observed input, but the same principles apply for N-D inputs. See Fig.\ref{fig:Structure_experiments} for an example 2-D structure.

Given a DLT layer, the layer above is constructed by multiple repetitions of the same kernel (see Sec.~\ref{sec:kernel}) at all spatial locations.
Assuming that the kernel has learned the distribution of specific object parts, the repeated application of the kernel accounts for all possible positions of these parts.
Starting from the observed nodes in the lowest layer, all layers above are then constructed recursively.

A naive dense structure for a 1-D input of size 7 (in layer 1) and kernels of size 3 is shown in Fig.~\ref{fig:Structure_cnn}.
Kernel weights at different spatial locations are shared, which is indicated by the edge colours.
This structure contains multiple parents for the inner nodes and thus induces loops for the BP message flow (Note: the arrows only indicate the conditional relationship between nodes, messages are passed in both directions).
Loops radically change the tractability of a graphical model:
inference is not exact any more and we would need to resort to costly approximate algorithms, like loopy BP or  contrastive divergence.
Even worse, many approximate inference algorithms do not have any convergence guarantees.
Instead, we design our model to form a tree, and thus reach guaranteed convergence with a single BP pass.

\begin{figure}[htbp]
    \centering
    \begin{minipage}[t]{.4\textwidth}
        \centering%
        \vspace{0pt}%
        \includegraphics[scale=0.45]{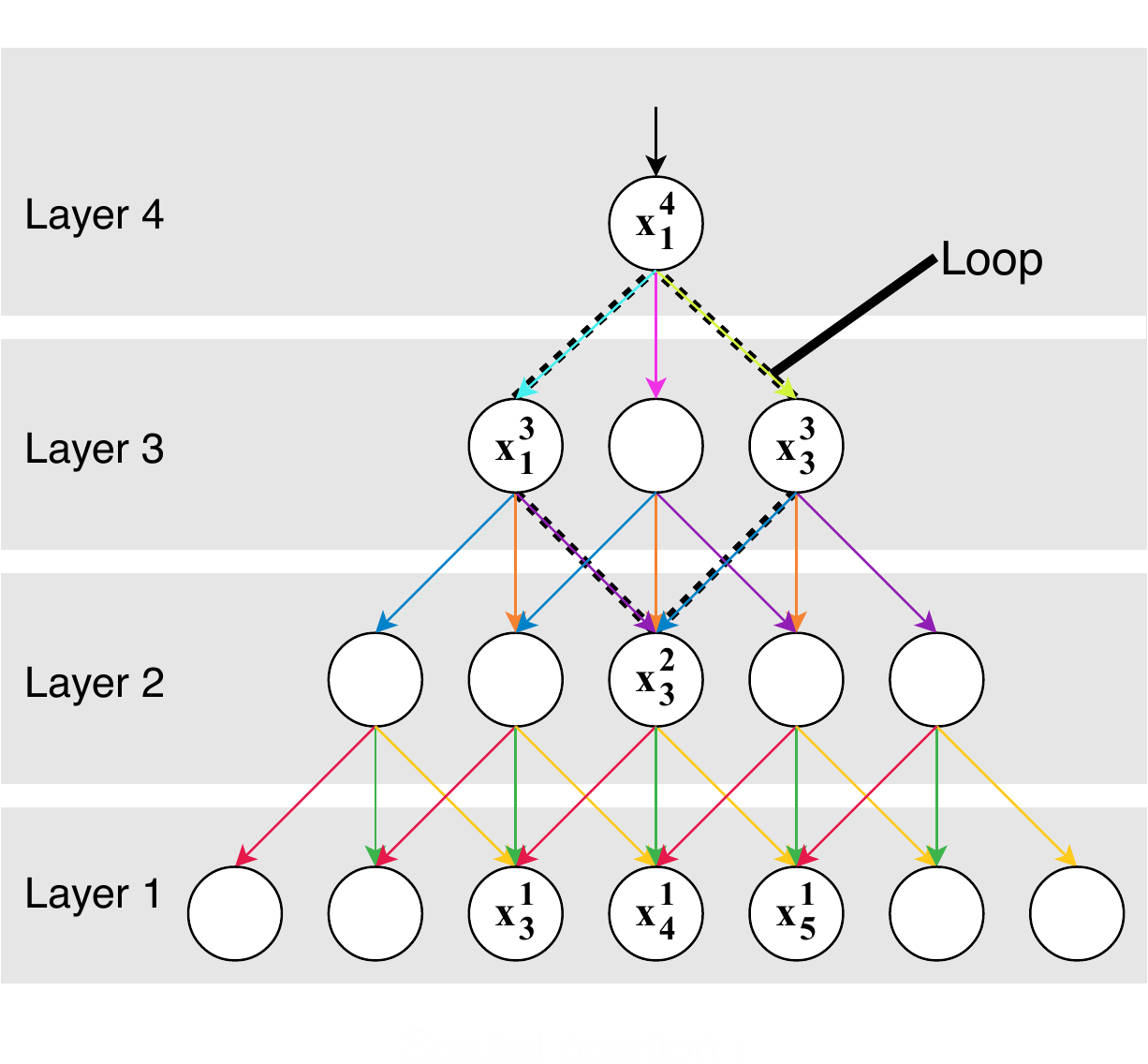}%
        \caption{
            Naive dense structure (similar to a CNN structure), \emph{contains loops}.
            Nodes depict random variables and arrows are conditional probability distributions.
            The kernels have size 3 and the colours indicate shared weights.
        }
        \label{fig:Structure_cnn}
    \end{minipage}\hspace{.025\textwidth}
    \begin{minipage}[t]{.575\textwidth}
        \centering%
        \vspace{0pt}%
        \includegraphics[scale=0.45]{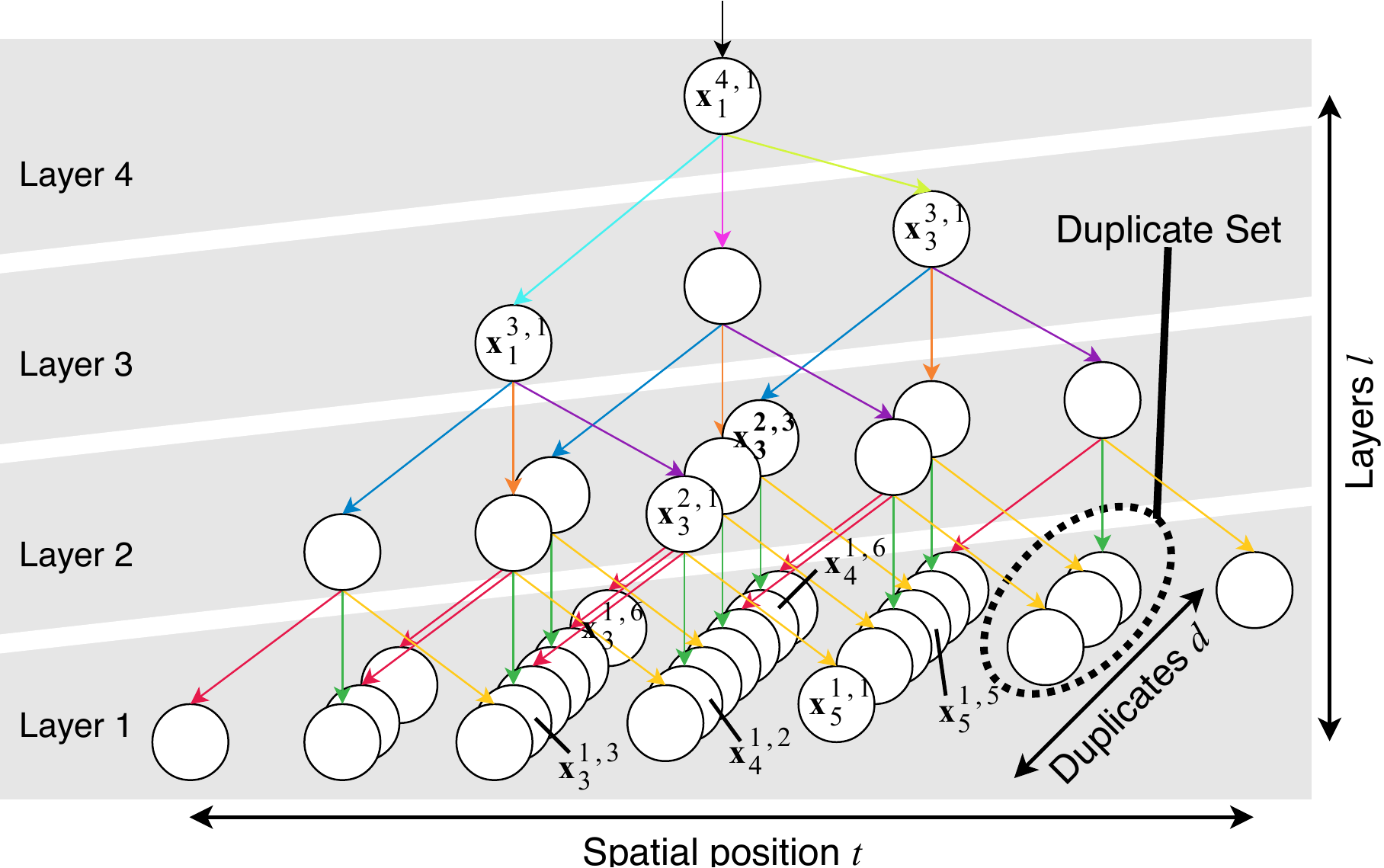}%
        \caption{Efficient DLT structure, \emph{does not contain loops}. The kernels have size 3 and edges with shared weights have the same colour. Nodes in a duplicate set overlap each other and are placed along a diagonal. There are 7 duplicate sets in layer 1, thus the number of unique inputs is the same as in Fig.~\ref{fig:Structure_cnn}.}
        \label{fig:Structure_clt}
    \end{minipage}
\end{figure}

In contrast to the naive example in Fig.~\ref{fig:Structure_cnn}, we design the DLT structure to avoid loops, while at the same time having a dense connectivity between layers.
Instead of connecting multiple kernels to the same input, we \emph{duplicate} the input node including the whole sub-tree below it
(i.e. including all ancestors, edges and input observations).
By doing this, each kernel gets its exclusive duplicate as input and thus no node has multiple parents.
As an example, take the loop in Fig.~\ref{fig:Structure_cnn} between the nodes $\mathbf{x}^4_1$, $\mathbf{x}^3_1$, $\mathbf{x}^3_3$ and $\mathbf{x}^2_3$.
This loop is caused by the multiple parents of node $\mathbf{x}^2_3$.
Instead of connecting $\mathbf{x}^3_1$ and $\mathbf{x}^3_3$ to the same node, each of them gets its own duplicate,
which are $\mathbf{x}^{2,1}_3$ and $\mathbf{x}^{2,3}_3$ respectively (see Fig.~\ref{fig:Structure_clt}).
Furthermore, duplicating the node $\mathbf{x}^{2,1}_3$ includes duplicating its whole sub-tree, and thus $\mathbf{x}^{2,1}_3$ gets its specific duplicate children $\mathbf{x}^{1,3}_3$, $\mathbf{x}^{1,2}_4$ and $\mathbf{x}^{1,1}_5$.
Analogously, $\mathbf{x}^{2,3}_3$ gets $\mathbf{x}^{1,6}_3$, $\mathbf{x}^{1,6}_4$ and $\mathbf{x}^{1,5}_5$ as children.

All nodes that are duplicates of each other are called a ``duplicate set''.
Crucially, each duplicate set shares the same kernel weights (indicated by the same colour of the edges).
This is necessary to keep the number of parameters constant and it leads to an efficient inference algorithm, see Sec.~\ref{sec:likelihood}. 
Fig.~\ref{fig:Structure_clt} shows a DLT for a 1-D input of size 7 and kernels of size 3. Note that all nodes within a duplicate set in layer 1 receive the same observed data as input (and thus there are only 7 unique inputs), but this constraint is \emph{not} encoded in the model.

Let $\mathbf{x}_{t}^{l,d}$ be the $d$th duplicate node at spatial
position $t$ in layer $l$, with $l\in\{1,...,L\}$, $t\in\{1,...,T^{l}\}$,
$d\in\{1,...,D_{t}^{l}\}$ (for brevity, we write $l\in L$ instead).
$L$ is the number of layers, $T^{l}$ is the spatial size of layer
$l$ and $D_{t}^{l}$ is the number of duplicates at position $t$
in layer $l$. Hence, the total number of spatial positions is $T=\sum_{l}T^{l}$ and the total number of nodes is $D=\sum_{l}\sum_{t\in T^{l}}D_{t}^{l}$.
Furthermore, let $K^l$ be the spatial size of the kernel in layer $l$.

Take for example a 1-D DLT with a constant kernel size of $K^{l}=4$ and stride%
\footnote{We denote the spatial step size of the kernel as ``stride'', in analogy to CNNs.}
2.
The described DLT structure leads to quadratic growth for the number of nodes $D$ in respect to a 1-D input image with $N$ pixels:
Given this structure, there are closed form solutions for $T^{l}=3(2^{L-l})-2$,
$T=3(2^{L})-2L-3$ and $D=\frac{1}{3}(2^{2L}-1)$
(for the derivation see Appx.~\ref{sec:derivation-of-closed-forms} in the supplement).
We need a spatial size of $N$ for the pixel input in layer $1$, i.e. $N=T^{1}=3(2^{L-1})-2$.
Thus the number of layers is $L=1+\log\frac{N+2}{3}$ and the total number of nodes is quadratic in $N$ with
$D=\frac{1}{3}(2^{2L}-1)=\frac{4}{27}\left(N+2\right)^{2}-\frac{1}{3}\in\mathcal{O}(N^{2})$.
This is problematic for reasonably large image inputs.
In the next section we develop an inference algorithm that is linear in the number of inputs $N$  for both runtime and memory complexity.

\subsection{Likelihood}\label{sec:likelihood}
This section describes how to calculate the marginal likelihood $\mathcal{L}$ for a given observation of $\mathbf{O}$ and fixed weights.
Following that each node (except the root) has exactly one parent, the joint likelihood of all nodes $\mathbf{X} = (\mathbf{O}, \mathbf{H}$) is given by
\begin{equation}
p(\mathbf{X})=\prod_{\mathbf{x}\in\mathbf{X}}p(\mathbf{x}|\mathbf{x}_{\textrm{par}})
\end{equation}
where $\mathbf{x}_{\textrm{par}}$ is the parent of $\mathbf{x}$.
Thus the model forms a latent tree, and the exact marginal likelihood
can be calculated by a single BP pass through
the tree \cite{pearl1988probabilistic}. We chose to start the BP from the leaves up to the root.

We denote the incoming BP message into node $\mathbf{x}_{t}^{l,d}$
as $\mathbf{m}_{t}^{l,d}=p(\mathbf{O}_{\mathbf{x}_{t}^{l,d}}|\mathbf{x}_{t}^{l,d})$,
where $\mathbf{m}_{t}^{l,d}=(m_{t,1}^{l,d},...,m_{t,F^{l}}^{l,d})$,
$F^{l}$ is the number of states for all nodes in layer $l$ and $\mathbf{O}_{\mathbf{x}}=\{\mathbf{o}|\mathbf{o}\in\mathbf{O}\textrm{ and }\mathbf{o}\textrm{ is ancestor of }\mathbf{x}\}$.
The BP is initialised by setting all messages of the input layer $\mathbf{m}_{t}^{1,d}$
to the respective observations for $\mathbf{x}_{t}^{1,d}\in\mathbf{O}$.
We assume a uniform distribution for all unobserved inputs $\mathbf{x}_{t}^{1,d}\notin\mathbf{O}$.
Then the BP recursively calculates all $\mathbf{m}_{t}^{l,d}$ of
the layers above, until we reach the root $\mathbf{m}_{1}^{L,1}$(for
the root $t=d=1$, since it has no duplicates and only a single spatial
position).
Assuming $w_{f}^{L}$ are the prior weights of the root, the marginal likelihood can be calculated
from the message at the top layer $L$:
\begin{equation}
\mathcal{L}=\sum_{f}w_{f}^{L}m_{1,f}^{L,1}
\end{equation}
In total, this algorithm would need to calculate the $\mathbf{m}_{t}^{l,d}$
messages $\forall l\in L$,$\forall t\in T^{l}$,$\forall d\in D_{t}^{l}$,
resulting in $D$ calculations.
We note that the message for a certain node depends only on the sub-tree
defined by the same node, which includes the connected nodes
and weights from the layers below and not the ones above. By design,
sub-trees defined by nodes from the same duplicate set have the \emph{same structure}, the \emph{same weights} and receive the \emph{same observations}
at their leaves.
Thus, each duplicate must receive the \emph{same message}:

\begin{restatable}{theorem}{thmduplicates}
\label{thm:duplicates}
The messages $\mathbf{m}_{t}^{l,d}$ for all duplicate nodes are equivalent to the unique message $\mathbf{u}_{t}^{l}$,
i.e. $\forall l\in L$,$\forall t\in T^{l}$,$\forall d\in D_{t}^{l}$:
$\mathbf{m}_{t}^{l,d}=\mathbf{u}_{t}^{l}$.
\end{restatable}

See Appx.~\ref{sec:proof-of-theorem-duplicates} for the proof. 
This means that only the unique messages $\mathbf{u}_{t}^{l}$ need to be calculated and stored.
Their number is equivalent to all spatial positions $T=\sum_{l\in L}T^{l}$.
Moreover, there is no further memory needed per node, only per unique message.
Thus both, runtime and memory complexity is in $\mathcal{O}(T)$.
In the special case of a 1-D DLT with a constant kernel size of $K^{l}=4$ and stride 2 (see the last paragraph of Sec.~\ref{sec:structure} and Appx.~\ref{sec:derivation-of-closed-forms}),
the number of unique messages $T$ grows linear in the number of inputs $N$, with $T=3(2^{L})-2L-3=2(N+2)\in\mathcal{O}(N)$.

Fig.~\ref{fig:Structure_messages} visualizes the message passing algorithm, including the unique
messages $\mathbf{u}_{t}^{l}$ and duplicate messages $\mathbf{m}_{t}^{l,d}$.
Assuming $w_{k,f,g}^{l-1}$ is the shared weight for all connected node pairs with the
parent $\mathbf{x}^{l,d}_t$ from layer $l$ and its $k$th child $\mathbf{x}^{l-1,e}_{t+k}$ from layer $l-1$ (where $k \in K^{l-1}$, $e\in D^{l-1}_t$, $g \in F^l$ and $f \in F^{l-1}$).
Thus $w_{k,f,g}^{l-1} = p(\mathbf{x}^{l-1,e}_{t+k} = f | \mathbf{x}^{l,d}_t = g)$
and all unique messages can be calculated as:
\begin{equation}
\label{eq:unique_messages}
u_{t,g}^{l}=\prod_{k}\sum_{f}w_{k,f,g}^{l-1}u_{t+k,f}^{l-1}
\end{equation}
\begin{figure}[htbp]
    \centering
    \includegraphics[scale=0.45]{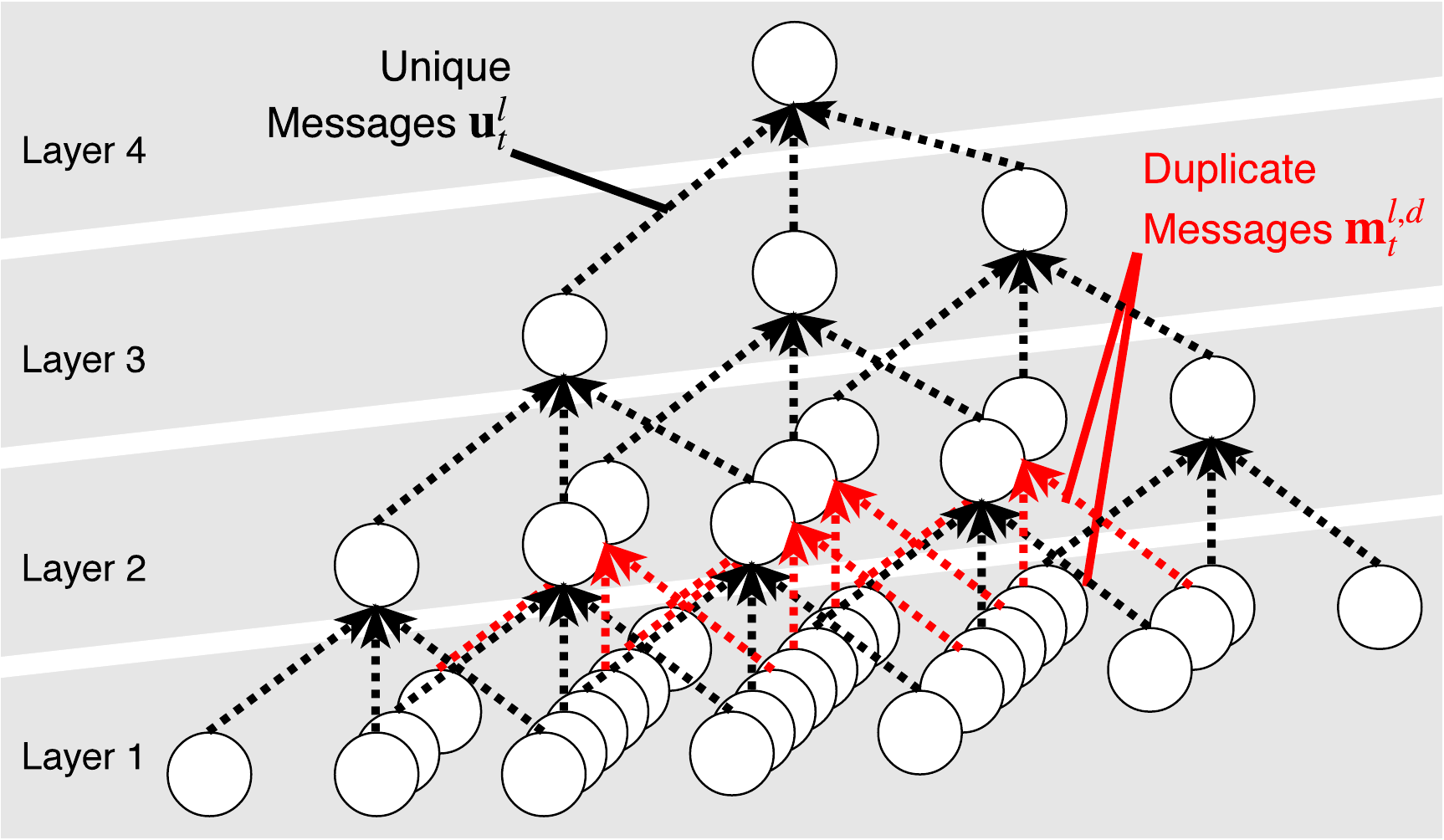}
    \caption{
        Message passing for the DLT marginal likelihood.
        Unique messages $\mathbf{u}_{t}^{l}$ that need to be calculated are shown in black.
        The red duplicate messages $\mathbf{m}_{t}^{l,d}$ have the same content, and thus no processing and no memory is needed for them.
    }
    \label{fig:Structure_messages}
\end{figure}

\subsection{Learning}\label{sec:learning}

Previous approaches for learning LT weights have applied an EM-procedure (e.g. \cite{Coi2011Learning}), where the bottom-up message passing is followed by a top-down pass to obtain the pairwise posteriors for each edge, which is then used to adjust the edge weights.
However, in the DLT case, top-down messages are not shared between nodes in a duplicate set, since the parent of each node is different.
Thus the number of top-down messages is in $\mathcal{O}(D)$, making a top-down pass intractable.

Instead, we follow a different approach:
we observe that the likelihood $\mathcal{L}$ can be calculated in a feed-forward way from the unique messages,
using only differentiable operations as defined in Eq.~\ref{eq:unique_messages}.
For numerical stability we target $\log \mathcal{L}$, and optimise it via Stochastic gradient ascent (SGA).
This also incorporates the advantages of SGA in general, like training on arbitrary large datasets.

There is one caveat: during training, we need to ensure that the weights form a valid distribution, i.e. $w_{k,f,g}^{l} \geq 0$ and $\sum_{f}w_{k,f,g}^{l}=1$ ($\forall l,k,g$).
We achieve this by re-parametrizing the model to use scores $s_{k,f,g}^{l}$ instead and defining the weights as $w_{k,f,g}^{l}=\frac{\exp(s_{k,f,g}^{l})}{\sum_{f}\exp(s_{k,f,g}^{l})}$.

\subsection{Sampling}\label{sec:sampling}
We are interested in sampling all nodes $\mathbf{X}$, conditioned on observations $\mathbf{O}\subset\mathbf{X}$.
$\mathbf{O}=\emptyset$ is possible, in which case we sample from the prior.
This is straightforward for a conventional LT:
first do an inference pass with the observations and then sample all nodes recursively, starting with the root.
Besides the problem that this would involve $\mathcal{O}(D)$ operations for DLT,
this would lead to diverging results for all duplicates.
During training, the duplicates were tied through the bottom-up message from the observed nodes, i.e. the input images.
However, there is no consistency enforced when sampling top-down,
and thus each duplicate can be in a different state.
In terms of the image pixels (i.e. the lowest layer of the model), this would lead to multiple sampled states for the same pixel.

In order to avoid these inconsistencies, we sample under the constraint
that all duplicates must take the same state. Thus we only obtain
a single sample $\mathbf{z}_{t}^{l}$ for all duplicate nodes $\mathbf{x}_{t}^{l,d}$.
This leads to the following sampling strategy: First, do an inference
pass to obtain the $\mathbf{u}_{t}^{l}$. Then start with the root
$\mathbf{x}_{1}^{L,1}$and take a sample $\mathbf{z}_{1}^{L}=(z_{1,1}^{L},...,z_{1,F^{L}}^{L})$,
$z_{1,f}^{L}\in\{0,1\}$, $\sum_{f}z_{1,f}^{L}=1$ from the posterior
$p(\mathbf{x}_{1}^{L,1}|\mathbf{O})$:
\begin{equation}
\mathbf{z}_{1}^{L}\sim\textrm{Cat}\left(\frac{u_{1,f}^{L}w_{f}^{L}}{\sum_{f}u_{1,f}^{L}w_{f}^{L}}\right)
\end{equation}
Furthermore, each $\mathbf{x}_{t}^{l,d}$ from a lower layer $l$
is sampled from the posterior $p(\mathbf{x}_{t}^{l,d}|\mathbf{y}_{t}^{l,d},\mathbf{O})$,
where $\mathbf{y}_{t}^{l,d}$ is the parent of $\mathbf{x}_{t}^{l,d}$
(and thus corresponds to one of the samples $\mathbf{z}_{t}^{l+1}$). Since
we constrain all samples to have the same state it holds $\forall d\in D_{t}^{l}:\mathbf{x}_{t}^{l,d}=\mathbf{x}_{t}^{l}$,
and thus it is sufficient to take a single sample from the product
of posteriors: $\mathbf{z}_{t}^{l}\sim\prod_{d}p(\mathbf{x}_{t}^{l}|\mathbf{y}_{t}^{l,d},\mathbf{O})$.
Furthermore, there are $K^{l}$ parents with different spatial locations
at the positions $t-k$ ($k\in K^{l}$), each of them having $D_{t-k}^{l+1}$
duplicates. Each of the parent duplicates is in the same state $\mathbf{z}_{t-k}^{l+1}$,
which leads to $\mathbf{z}_{t}^{l}\sim\prod_{k}p(\mathbf{z}_{t}^{l}|\mathbf{z}_{t-k}^{l+1},\mathbf{O})^{D_{t-k}^{l+1}}$.

Given all samples of the layer above $\mathbf{z}_{t}^{l+1}$, we first
calculate the product of the conditional posteriors $\boldsymbol{\pi}_{t}^{l}=(\pi_{t,1}^{l},...,\pi_{t,F^{l}}^{l})$
of all $D_{t}^{l}=\sum_{k}D_{t-k}^{l+1}$ parents of the set. Then
$\mathbf{z}_{t}^{l}$ is sampled from the product distribution:
\begin{align}
\pi_{t,f}^{l}= & \prod_{k}\left(u_{t,f}^{l}\sum_{g}w_{k,f,g}^{l}z_{t-k,g}^{l+1}\right)^{D_{t-k}^{l+1}}\\
\mathbf{z}_{t}^{l}\sim & \textrm{Cat}\left(\frac{\pi_{t,f}^{l}}{\sum_{f}\pi_{t,f}^{l}}\right)
\end{align}
where $g\in F^{l+1}$and $f\in F^{l}$. This sampling step is repeated
for all spatial positions $t$ and layers $l$, until the inputs $\mathbf{z}_{t}^{1}$
have been sampled. As for inference, the complexity is in $\mathcal{O}(T)$.

\section{Experiments}
\label{sec:experiments}

The goal of this section is to verify that the DLT is able to learn a hierarchical distribution of image parts from incomplete training data. We chose in-painting with randomly missing patches as example task, where the training data is corrupted via missing patches from the same distribution. We compare the DLT performance to (1) the same model trained on complete data and (2) other common generative and discriminative models. Finally, we visualise the learned hierarchy of image parts.

MNIST \cite{lecun1998mnist} digits and Fashion-MNIST \cite{xiao2017fashion} clothing items have been chosen as example data, due to its simple yet sufficiently rich structure to learn a hierarchy of parts.
The images have a resolution of 28x28 and we remove a patch of 12x12 within each image of the train and test sets.
The missing location is chosen randomly from a uniform distribution for each image.
DLT model categorical distributions as input, and thus we discretise each pixel into black and white for MNIST and 16 grey-scale values for Fashion-MNIST.
To measure the in-painting performance, we scale the pixel values to the range $[0,1]$ and then report the mean squared error (MSE) between the original and in-painted pixels within the missing patches.
Note that the class label is not used in the experiments.

We construct a 4 layer DLT, as described in Fig.~\ref{fig:Structure_experiments}.
The input images require a spatial resolution $T^1$ of 28x28 and 2 (for MNIST) or 16 (for Fashion-MNIST) states in layer 1.
Each subsequent layer has a lower spatial size, until we reach the single root, with size 1x1.
The first two kernels have a stride of 2, which leads to a reduced size of the upper layers.
The kernel sizes $K^l$ of 4x4 and 5x5 have been chosen to result in a single root at layer 4, while also providing sufficient overlap between kernels.
We increase the number of states $F^l$ in the upper layers, in order to learn increasingly complex parts. This leads to a total of 1,691,648 parameters for the kernel weights.
The DLT includes 10,651 random variables, while only 979 unique messages need to be calculated to obtain the likelihood.
The model is implemented in tensorflow \cite{Abadi2015} and trained using SGA for 100 epochs on the 60,000 training images, while using a batch size of 128.
A single inference pass takes 71ms per batch, and sampling takes 41ms per sample on a Nvidia GTX 1080Ti.
For in-painting, we fill the missing patch by taking a single sample from the model,
conditioned on the observed parts.

\begin{figure}[htbp]
    \centering
    \includegraphics[width=\textwidth]{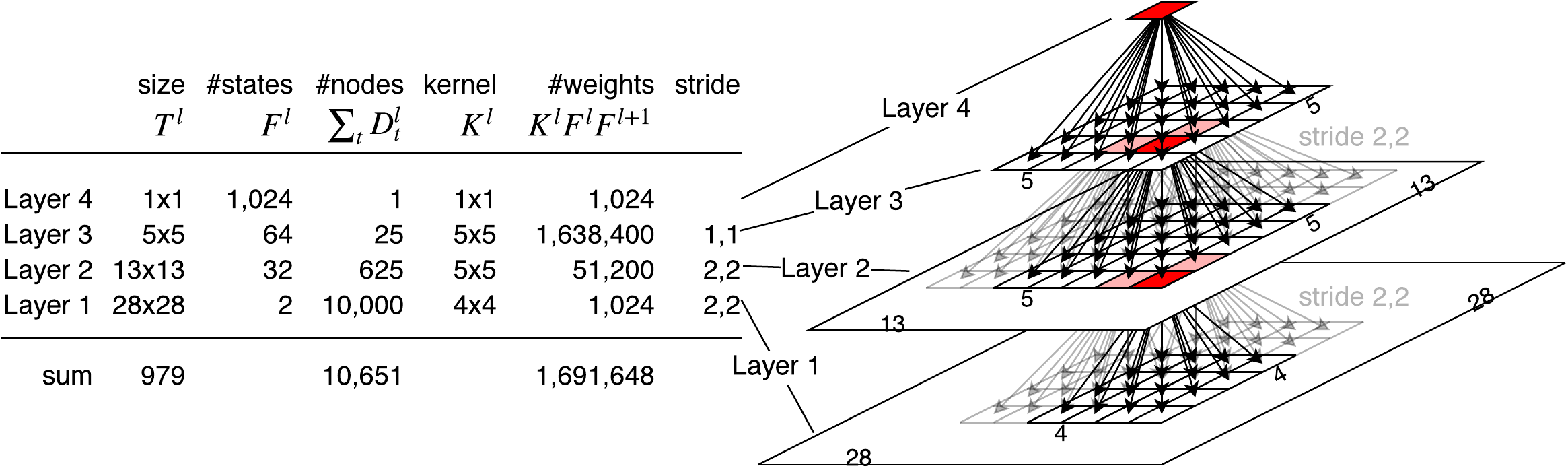}
    \caption{
        Structure of the 2-D DLT used for the MNIST experiments.
        On the right, each tile represents a spatial location, and thus there are potentially many duplicate nodes per tile.
        The conditional dependencies for the parents in red are shown as arrows, but the same pattern is repeated for every parent in every layer.
        (Note: the Fashion-MNIST model has 16 states in layer 1)
    }
    \label{fig:Structure_experiments}
\end{figure}

We compare our DLT to a Quad-tree, RBM and CNN of similar architecture.
The Quad-tree has 4 layers with shared weights and the same number of states as the DLT per layer.
Each of the nodes in layers 2 and 3 has 2x2 children (as a regular Quad-tree),
but the root has 7x7 children in order to evenly connect to a single root.
The RBM has a fully connected hidden layer with 4,096 states and missing values are handled according \cite{Hinton2010}.
The CNN is an Auto-encoder U-net with skip connections, with an architecture inspired by \cite{ronneberger2015u}.
We set the architecture of the contracting part to mirror the DLT, via selecting the same number of layers, the same kernel sizes and strides, and choosing the number of channels to be equivalent to the respective number of DLT states.
The expanding part mirrors the contracting one, using the convolution transpose operation.
The CNN (complete) is trained with incomplete images as input and the corresponding complete image as target.
In order to train the CNN with incomplete targets, we use an auto-encoder loss on the known parts of the image.

After training the DLT on incomplete MNIST, it achieves a negative log-likelihood of 1,418.56 on the test set,
which is only slightly larger than for training on complete MNIST with 1,406.78.
In comparison, the Quad-tree reaches 69.85 on incomplete and 62.18 on complete data.
The difference in magnitude is due to the size of the input:
The DLT has 10,000 observed nodes, while the Quad-tree has only 784.

Tab.~\ref{tab:inpainting_results} shows the in-painting results. 
All generative models (DLT, Quad-tree and RBM) have almost the same performance on complete and incomplete data.
This confirms the intuition that an explicit model of the input distribution helps to fill in randomly missing data.
The CNN does well as a discriminative model trained on complete data, but fails in dealing with missing parts.
The Quad-tree suffers from block-artefacts, while the DLT can produce smooth images due to the overlapping kernels.
The RBM completions are noisy, possibly due to the non-hierarchical structure.
Further qualitative results are shown in Appx.~\ref{sec:sampling-results-comparison}.

\newcommand{\raiseamount}{-0.25\height}
\newcommand{\resultscale}{0.525}
\begin{table}[htbp]
    \centering
    \caption{In-painting MSE results for different models trained on complete and incomplete data, for both, the MNIST and Fashion-MNIST datasets. Additionally, 6 example predictions for incomplete data are shown on the right. The first row shows the ground-truth, the second one shows the observed part (with the missing part in grey).}%
    \label{tab:inpainting_results}%
    \setlength{\tabcolsep}{2pt}%
    \begin{tabular}{lcccccccc}
        & \multicolumn{2}{c}{ MNIST } & & \multicolumn{2}{c}{ Fashion-MNIST } & & MNIST & Fashion-MNIST \\
        & compl. & incompl. & & compl. & incompl. & & incomplete & incomplete \\
        truth & & & & & & & \raisebox{\raiseamount}{\includegraphics[scale=\resultscale]{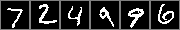}} & \raisebox{\raiseamount}{\includegraphics[scale=\resultscale]{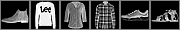}} \\
        observed & & & & & & & \raisebox{\raiseamount}{\includegraphics[scale=\resultscale]{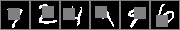}} & \raisebox{\raiseamount}{\includegraphics[scale=\resultscale]{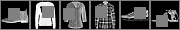}} \\
        DLT & .1496 & \textbf{.1551} & & .0650 & \textbf{.0645} & & \raisebox{\raiseamount}{\includegraphics[scale=\resultscale]{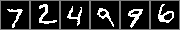}} & \raisebox{\raiseamount}{\includegraphics[scale=\resultscale]{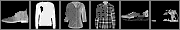}} \\
        Quad-tree & .1810 & .1817 & & .0715 & .0721 & & \raisebox{\raiseamount}{\includegraphics[scale=\resultscale]{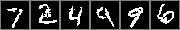}} & \raisebox{\raiseamount}{\includegraphics[scale=\resultscale]{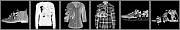}} \\
        RBM & .2002 & .2110 & & .0778 & .0859 & & \raisebox{\raiseamount}{\includegraphics[scale=\resultscale]{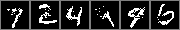}} & \raisebox{\raiseamount}{\includegraphics[scale=\resultscale]{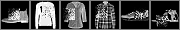}} \\
        CNN & \textbf{.1082} & .2366 & & \textbf{.0294} & .0723 & & \raisebox{\raiseamount}{\includegraphics[scale=\resultscale]{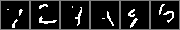}} & \raisebox{\raiseamount}{\includegraphics[scale=\resultscale]{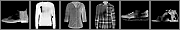}} \\
    \end{tabular}
\end{table}

Fig.~\ref{fig:parts_per_layer} shows how the DLT composes MNIST images into a hierarchy of parts.
To visualise a possible state $f$ of layer $l$, we create a new DLT from the sub-graph that has a single node from layer $l$ as root,
and setting its prior deterministic to state $f$.
Each state is then represented by a sampled image from this sub-graph.
E.g. for the third layer, this is a DLT with a single node in layer 3, 5x5 nodes in layer 2 and 12x12 spatial position in layer 1. This leads to 64 sampled images of size 12x12, one for each state.
In the top row of Fig.~\ref{fig:parts_per_layer}, we observe that layer 2 combines pixels into edge kernels at different translations and orientations.
Layer 3 learns to combine the edges of layer 2 into parts of digits.
It includes lines at different angles, curved lines, but also circles and other shapes.
Layer 4 combines these parts into various types of complete digits.

In the bottom row, we show how a specific image is sampled. We start with the root at layer 4, where one state out of 1024 is randomly selected from the prior (shown via the green arrow).
Then each of the 5x5 nodes in layer 3 is conditioned on layer 4, selecting one of its 64 possible states.
The conditional dependency is highlighted via the red boxes and arrows.
Note that the samples in layer 3 are presented on a 5x5 grid without overlap, but when combined in layer 1 there is significant overlap between the receptive fields of the upper layers.
We observe that the parts from layer 3 are sampled to form a coherent representation of the edges present in the digit 7.
Then the same process is repeated for layers 2 and 1. Fig.~\ref{fig:parts_per_layer_fashion} in Appx.~\ref{sec:sampling-results-comparison} shows the equivalent results for Fashion-MNIST.

\begin{figure}[htbp]
    \centering
    \includegraphics[width=\textwidth]{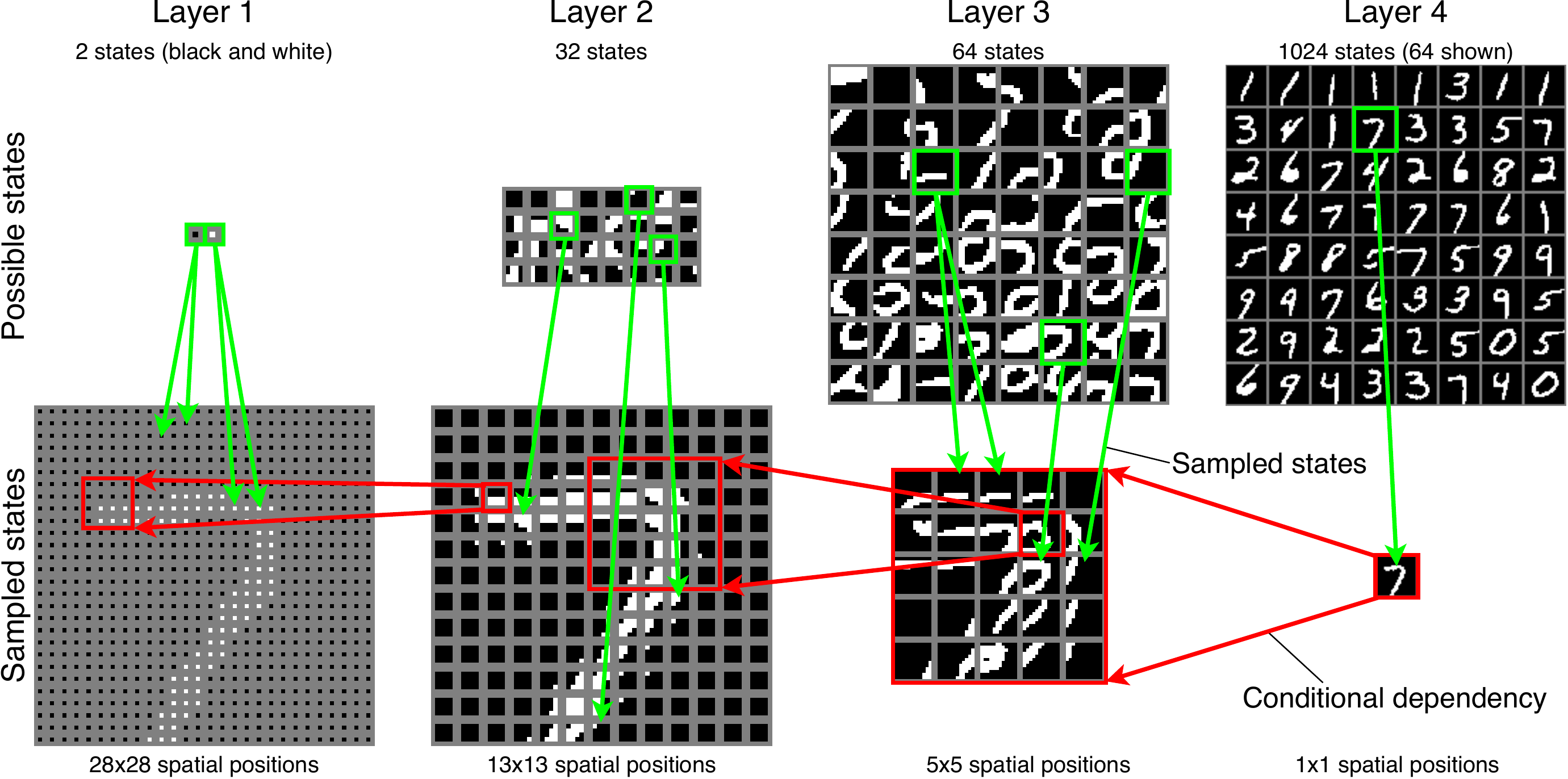}
    \caption{
        \emph{Top row:} learned DLT parts per state per layer, when trained on incomplete MNIST.
        \emph{Bottom row:} visualisation of the sampling process.
    }
    \label{fig:parts_per_layer}
\end{figure}

\section{Conclusion}
We have introduced the Dense Latent Tree and demonstrated its ability to learn a hierarchical composition of images into parts, when only incomplete data is available.
DLTs allow for efficient inference and sampling due to its tree structure, while incorporating a dense connectivity with overlapping parts that is more similar to CNNs.
Future directions may explore different connectivity patterns, like dilated convolutions and skip connections, as well as extensions to other non-categorical distributions.

{\small
    \bibliographystyle{plain}
    \bibliography{library}

\begin{thebibliography}{10}

\bibitem{Abadi2015}
Mart{\'{i}}n Abadi~et. al.
\newblock {TensorFlow: Large-Scale Machine Learning on Heterogeneous
  Distributed Systems}.
\newblock Technical report, 2015.

\bibitem{bouman1994multiscale}
Charles~A Bouman and Michael Shapiro.
\newblock {A multiscale random field model for Bayesian image segmentation}.
\newblock {\em IEEE Transactions on image processing}, 3(2):162--177, 1994.

\bibitem{Coi2011Learning}
Myung~Jin Choi, Vincent Y~F Tan, Animashree Anandkumar, and Alan~S Willsky.
\newblock {Learning Latent Tree Graphical Models}.
\newblock {\em J. Mach. Learn. Res.}, 12:1771--1812, jul 2011.

\bibitem{eslami2012generative}
S~Eslami and Christopher Williams.
\newblock {A generative model for parts-based object segmentation}.
\newblock In {\em Advances in Neural Information Processing Systems}, pages
  100--107, 2012.

\bibitem{eslami2014shape}
S~M~Ali Eslami, Nicolas Heess, Christopher K~I Williams, and John Winn.
\newblock {The shape boltzmann machine: a strong model of object shape}.
\newblock {\em International Journal of Computer Vision}, 107(2):155--176,
  2014.

\bibitem{goodfellow2014generative}
Ian Goodfellow, Jean Pouget-Abadie, Mehdi Mirza, Bing Xu, David Warde-Farley,
  Sherjil Ozair, Aaron Courville, and Yoshua Bengio.
\newblock {Generative adversarial nets}.
\newblock In {\em Advances in Neural Information Processing Systems}, pages
  2672--2680, 2014.

\bibitem{Hinton2010}
Geoffrey Hinton.
\newblock {A Practical Guide to Training Restricted Boltzmann Machines}.
\newblock Technical report, University of Toronto, 2010.

\bibitem{kingma2013auto}
Diederik~P Kingma and Max Welling.
\newblock {Auto-encoding variational bayes}.
\newblock In {\em ICLR}, 2014.

\bibitem{lazarsfeld1968latent}
Paul~Felix Lazarsfeld and Neil~W Henry.
\newblock {\em {Latent Structure Analysis}}.
\newblock Houghton Mifflin, 1968.

\bibitem{lecun1998mnist}
Yann LeCun.
\newblock {The MNIST database of handwritten digits}.
\newblock {\em http://yann. lecun. com/exdb/mnist/}, 1998.

\bibitem{pearl1988probabilistic}
Judea Pearl.
\newblock {\em {Probabilistic reasoning in intelligent systems: networks of
  plausible inference}}.
\newblock Morgan Kaufmann, 1988.

\bibitem{ronneberger2015u}
Olaf Ronneberger, Philipp Fischer, and Thomas Brox.
\newblock {U-net: Convolutional networks for biomedical image segmentation}.
\newblock In {\em International Conference on Medical image computing and
  computer-assisted intervention}, pages 234--241. Springer, 2015.

\bibitem{Salakhutdinov2009Deep}
Ruslan Salakhutdinov and Geoffrey Hinton.
\newblock {Deep Boltzmann Machines}.
\newblock In {\em Aistats}, volume~1, pages 448--455, 2009.

\bibitem{slorkey2003image}
A~J Slorkey and Christopher K~I Williams.
\newblock {Image modeling with position-encoding dynamic trees}.
\newblock {\em Pattern Analysis and Machine Intelligence, IEEE Transactions
  on}, 25(7):859--871, 2003.

\bibitem{sudderth2005learning}
Erik~B Sudderth, Antonio Torralba, William~T Freeman, and Alan~S Willsky.
\newblock {Learning hierarchical models of scenes, objects, and parts}.
\newblock In {\em Computer Vision, 2005. ICCV 2005. Tenth IEEE International
  Conference on}, volume~2, pages 1331--1338. IEEE, 2005.

\bibitem{van2016conditional}
Aaron van~den Oord, Nal Kalchbrenner, Lasse Espeholt, Oriol Vinyals, Alex
  Graves, and Others.
\newblock {Conditional image generation with pixelcnn decoders}.
\newblock In {\em Advances in Neural Information Processing Systems}, pages
  4790--4798, 2016.

\bibitem{xiao2017fashion}
Han Xiao, Kashif Rasul, and Roland Vollgraf.
\newblock {Fashion-mnist: a novel image dataset for benchmarking machine
  learning algorithms}.
\newblock {\em arXiv preprint arXiv:1708.07747}, 2017.

\end{thebibliography}
}

\clearpage
\section*{Appendices for ``Imagining the Unseen: Learning a Distribution over Incomplete Images with Dense Latent Trees''}
\appendix
\section{Derivation of Closed Forms for $T^{l}$, $T$ and $D$}
We are interested in obtaining a closed form solution for the number of spatial positions per layer $T^{l}$, the total number of spatial positions $T$ and the total number of nodes $D$, given a constant kernel size $K^l=4$ and stride 2. Thus each kernel has four children and is repeated at every second position within each layer. This leads to a structure as shown in Fig.~\ref{fig:Structure_complexity} (unrolled for $L=5$ layers).
\label{sec:derivation-of-closed-forms}
\begin{figure}[htbp]
    \centering
    \includegraphics[width=\textwidth]{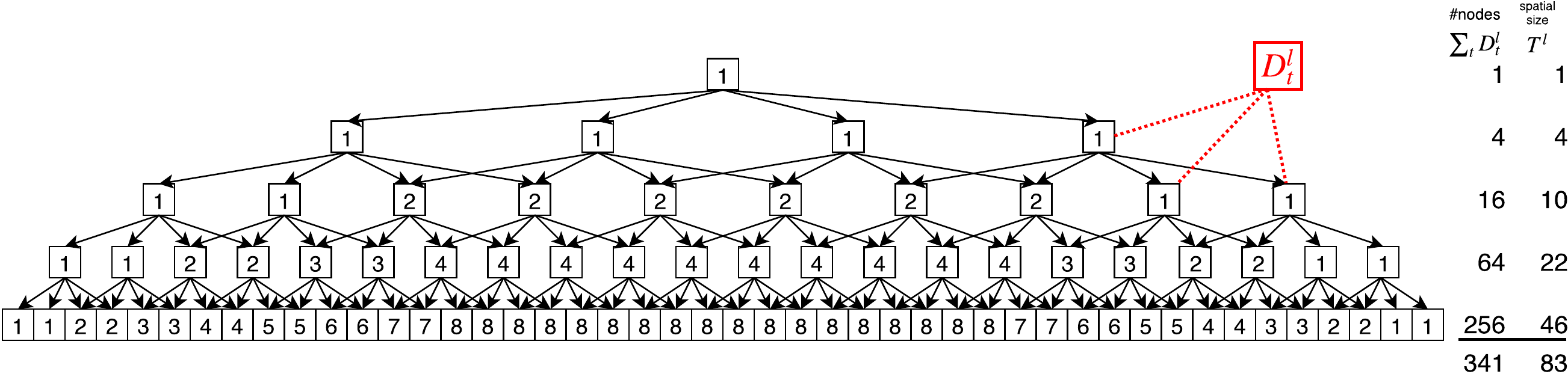}
    \caption{
        DLT with constant $K^l=4$ and stride 2. The number of duplicates $D^l_t$ is shown inside the square at each spatial position.
        For this graph, $D=\sum_{l}\sum_{t}D_{t}^{l}=341$ and $T=\sum_{l}T^{l}=83$.
    }
    \label{fig:Structure_complexity}
\end{figure}

For the derivation of these quantities, it is easier to number the
layers in reverse order starting from the root, i.e. $\hat{l}=L-l+1$.
We observe that $T^{\hat{l}}$ can be obtained via a recursive formula
as $T^{1}=1$ and $T^{\hat{l}}=2(T^{\hat{l}-1}+1)$, which leads to
a closed form of $T^{\hat{l}}=3(2^{\hat{l}-1})-2$, and thus:
\begin{equation}
T^{l}=3(2^{L-l})-2
\end{equation}
Solving for $T$ leads to:
\begin{equation}
T=\sum_{l}T^{l}=\sum_{l}3(2^{L-l})-2=3(2^{L})-2L-3
\end{equation}

Fig.~\ref{fig:Structure_complexity} shows the number of duplicates $D_{t}^{l}$ per spatial position.
We note that $D_{t}^{l}$ can be obtained by summing over its parents
as $D_{t}^{l}=\sum_{k}D_{t-k}^{l+1}$. Furthermore, the number of
nodes per layer can be derived as $\sum_{t}D_{t}^{\hat{l}}=2^{2(\hat{l}-1)}$.
Solving for $D$ leads to:
\begin{equation}
D=\sum_{l}\sum_{t}D_{t}^{l}=\sum_{\hat{l}}2^{2(\hat{l}-1)}=\frac{1}{3}(2^{2L}-1)
\end{equation}

\section{Proof of Theorem~\ref{thm:duplicates}}\label{sec:proof-of-theorem-duplicates}
\thmduplicates*
\begin{proof}
    We prove this via induction over $l$: Let
    $\mathbf{o}_{t}$ be the observed input at position $t$. By design,
    each of the input duplicates in layer 1 has the same observation,
    and thus $\forall t\in T^{1},\forall d\in D_{t}^{1}$: $\mathbf{o}_{t}=\mathbf{m}_{t}^{1,d}=\mathbf{u}_{t}^{1}$.
    Now assuming it holds that all duplicate messages in layer $l-1$
    have the same value, i.e. $\forall t\in T^{l-1},\forall d\in D_{t}^{l-1},\forall f\in F^{l-1}$:
    $m_{t,f}^{l-1,d}=u_{t,f}^{l-1}$. Then also all messages at layer
    $l$ must be identical: let $\left\{ \mathbf{m}_{t+k}^{l-1,d(k)}\right\} _{k\in K^{l-1}}$
    be the messages of the $K^{l-1}$ children of $\mathbf{x}_{t}^{l,d}$
    within layer $l-1$. Then $\forall t\in T^{l},\forall d\in D_{t}^{l},\forall g\in F^{l}$ we can calculate the BP message as:
    \begin{equation}
    m_{t,g}^{l,d}=\prod_{k}\sum_{f}w_{k,f,g}^{l-1}m_{t+k,f}^{l-1,d(k)}=\prod_{k}\sum_{f}w_{k,f,g}^{l-1}u_{t+k,f}^{l-1}=u_{t,g}^{l}
    \end{equation}
\end{proof}

\section{Additional Qualitative Results}
\label{sec:sampling-results-comparison}

Tab.~\ref{tab:inpainting_results_complete} shows the example predictions for complete data, corresponding to the results on incomplete data in Tab.~\ref{tab:inpainting_results}.
\begin{table}[htbp]
    \centering
    \caption{In-painting qualitative results for models trained on complete data. Shown are 6 example images. the first row shows the ground-truth, the second one shows the observed part (with the missing patch in grey) and the following rows show results for DLT, Quad-tree, RBM and CNN.}%
    \label{tab:inpainting_results_complete}%
    \setlength{\tabcolsep}{2pt}%
    \begin{tabular}{lcc}
        & MNIST & Fashion-MNIST \\
        & complete & complete \\
        truth & \raisebox{\raiseamount}{\includegraphics[scale=\resultscale]{inpainting_complete}} & \raisebox{\raiseamount}{\includegraphics[scale=\resultscale]{inpainting_fashion_random_complete_classnet_fashion_quantize16_missing}} \\
        observed & \raisebox{\raiseamount}{\includegraphics[scale=\resultscale]{inpainting_missing}} & \raisebox{\raiseamount}{\includegraphics[scale=\resultscale]{inpainting_fashion_random_missing_classnet_fashion_quantize16_missing}} \\
        DLT & \raisebox{\raiseamount}{\includegraphics[scale=\resultscale]{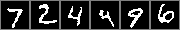}} & \raisebox{\raiseamount}{\includegraphics[scale=\resultscale]{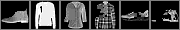}} \\
        Quad-tree & \raisebox{\raiseamount}{\includegraphics[scale=\resultscale]{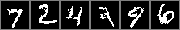}} & \raisebox{\raiseamount}{\includegraphics[scale=\resultscale]{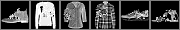}} \\
        RBM & \raisebox{\raiseamount}{\includegraphics[scale=\resultscale]{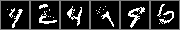}} & \raisebox{\raiseamount}{\includegraphics[scale=\resultscale]{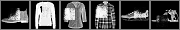}} \\
        CNN & \raisebox{\raiseamount}{\includegraphics[scale=\resultscale]{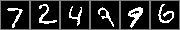}} & \raisebox{\raiseamount}{\includegraphics[scale=\resultscale]{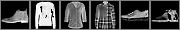}} \\
    \end{tabular}
\end{table}

Tab.~\ref{tab:inpainting_growing_results_all} shows the sampling result for different models trained on incomplete data, when a patch of varying size is observed.
In the first column on MNIST, nothing is observed and thus the DLT samples are from the prior only, which shows random digits.
In the next two columns, the model observes a few pixels in the upper left corner and hence samples random numbers, that are consistent with the observed parts.
From the 4th column on it becomes clear that the observed number must be a seven,
and thus the models samples the digit 7 with varying strokes in the lower part.
The Fashion-MNIST results show a similar pattern of randomness.
The outline and overall consistent grey-scale colouring of clothing items is visible. However, the fine texture is lost and the whole item looks smooth.

The Quad-tree shows a similar pattern than DLT, i.e. sampling random digits/clothes from the prior in the first 3 columns, but then converging towards the ground-truth.
The Quad-tree typical block-artefacts are clearly visible.

The RBM tends to sample the digit zero (for MNIST) and a jacket-like cloud (for Fashion-MNIST) when only few parts are observed. Starting with column 4, the convergence to the input image becomes visible.
Overall, the samples look more noisy than for the DLT and Quad-tree.

The CNN showed poor results when trained with incomplete targets (see Tab.~\ref{tab:inpainting_results}) and thus we include here the results for the CNN trained on complete targets instead.
This CNN had the best results with an MSE of .1082 for completing missing patches of size 12x12, i.e. in the case that test and training data corruption follows the same pattern.
In this experiment, the observed part is varying in size and thus different than in the training data.
In the first 3 columns, the CNN predicts the image as complete black. This is probably due to the non-generative nature of this CNN, which cannot sample from a prior.
In the 4th and 5th column, when the other models clearly pick up the tendency towards the target, the CNN is still not completing the image in a plausible way.
Thus the CNN does well for the exact task that it has been trained on, but poorly generalises to similar in-painting tasks.

\begin{table}[htbp]
    \centering
    \caption{Sampling results for an observed part of varying size. The first row shows the ground-truth image, the second one the observed parts, and the last 5 rows show different samples. The CNN output is deterministic, thus we show only one prediction.}
    \label{tab:inpainting_growing_results_all}
    \setlength{\tabcolsep}{2pt}
\begin{tabular}{lccc}
    \midrule
    & \multicolumn{3}{c}{ MNIST } \\
    \midrule
    & DLT (incomplete) & Quad-tree (incomplete) & RBM (incomplete) \\
    truth & \raisebox{\raiseamount}{\includegraphics[scale=\resultscale]{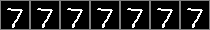}} & \raisebox{\raiseamount}{\includegraphics[scale=\resultscale]{inpainting_growing_complete_classnet_strided2_woclass_missing}} & \raisebox{\raiseamount}{\includegraphics[scale=\resultscale]{inpainting_growing_complete_classnet_strided2_woclass_missing}} \\
    observed & \raisebox{\raiseamount}{\includegraphics[scale=\resultscale]{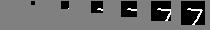}} & \raisebox{\raiseamount}{\includegraphics[scale=\resultscale]{inpainting_growing_missing_classnet_strided2_woclass_missing}} & \raisebox{\raiseamount}{\includegraphics[scale=\resultscale]{inpainting_growing_missing_classnet_strided2_woclass_missing}} \\
    sample 1 & \raisebox{\raiseamount}{\includegraphics[scale=\resultscale]{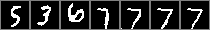}} & \raisebox{\raiseamount}{\includegraphics[scale=\resultscale]{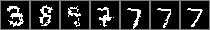}} & \raisebox{\raiseamount}{\includegraphics[scale=\resultscale]{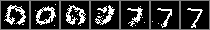}} \\
    sample 2 & \raisebox{\raiseamount}{\includegraphics[scale=\resultscale]{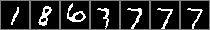}} & \raisebox{\raiseamount}{\includegraphics[scale=\resultscale]{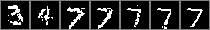}} & \raisebox{\raiseamount}{\includegraphics[scale=\resultscale]{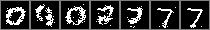}} \\
    sample 3 & \raisebox{\raiseamount}{\includegraphics[scale=\resultscale]{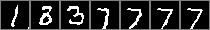}} & \raisebox{\raiseamount}{\includegraphics[scale=\resultscale]{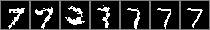}} & \raisebox{\raiseamount}{\includegraphics[scale=\resultscale]{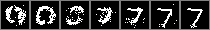}} \\
    sample 4 & \raisebox{\raiseamount}{\includegraphics[scale=\resultscale]{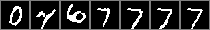}} & \raisebox{\raiseamount}{\includegraphics[scale=\resultscale]{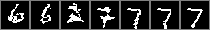}} & \raisebox{\raiseamount}{\includegraphics[scale=\resultscale]{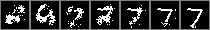}} \\
    sample 5 & \raisebox{\raiseamount}{\includegraphics[scale=\resultscale]{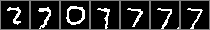}} & \raisebox{\raiseamount}{\includegraphics[scale=\resultscale]{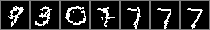}} & \raisebox{\raiseamount}{\includegraphics[scale=\resultscale]{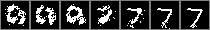}} \\
    & & & \\
    & CNN (complete) & & \\
    truth & \raisebox{\raiseamount}{\includegraphics[scale=\resultscale]{inpainting_growing_complete_classnet_strided2_woclass_missing}} & & \\
    observed & \raisebox{\raiseamount}{\includegraphics[scale=\resultscale]{inpainting_growing_missing_classnet_strided2_woclass_missing}} & & \\
    prediction & \raisebox{\raiseamount}{\includegraphics[scale=\resultscale]{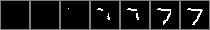}} & & \\
    & & & \\
    \midrule
    & \multicolumn{3}{c}{ Fashion-MNIST } \\
    \midrule
    & DLT (incomplete) & Quad-tree (incomplete) & RBM (incomplete) \\
    truth & \raisebox{\raiseamount}{\includegraphics[scale=\resultscale]{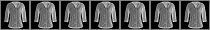}} & \raisebox{\raiseamount}{\includegraphics[scale=\resultscale]{inpainting_growing_fashion_complete_classnet_fashion_quantize16_missing}} & \raisebox{\raiseamount}{\includegraphics[scale=\resultscale]{inpainting_growing_fashion_complete_classnet_fashion_quantize16_missing}} \\
    observed & \raisebox{\raiseamount}{\includegraphics[scale=\resultscale]{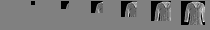}} & \raisebox{\raiseamount}{\includegraphics[scale=\resultscale]{inpainting_growing_fashion_missing_classnet_fashion_quantize16_missing}} & \raisebox{\raiseamount}{\includegraphics[scale=\resultscale]{inpainting_growing_fashion_missing_classnet_fashion_quantize16_missing}} \\
    sample 1 & \raisebox{\raiseamount}{\includegraphics[scale=\resultscale]{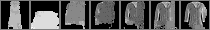}} & \raisebox{\raiseamount}{\includegraphics[scale=\resultscale]{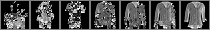}} & \raisebox{\raiseamount}{\includegraphics[scale=\resultscale]{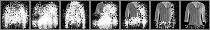}} \\
    sample 2 & \raisebox{\raiseamount}{\includegraphics[scale=\resultscale]{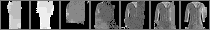}} & \raisebox{\raiseamount}{\includegraphics[scale=\resultscale]{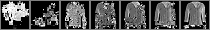}} & \raisebox{\raiseamount}{\includegraphics[scale=\resultscale]{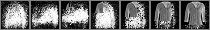}} \\
    sample 3 & \raisebox{\raiseamount}{\includegraphics[scale=\resultscale]{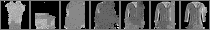}} & \raisebox{\raiseamount}{\includegraphics[scale=\resultscale]{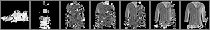}} & \raisebox{\raiseamount}{\includegraphics[scale=\resultscale]{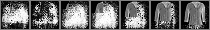}} \\
    sample 4 & \raisebox{\raiseamount}{\includegraphics[scale=\resultscale]{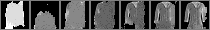}} & \raisebox{\raiseamount}{\includegraphics[scale=\resultscale]{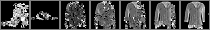}} & \raisebox{\raiseamount}{\includegraphics[scale=\resultscale]{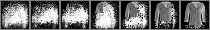}} \\
    sample 5 & \raisebox{\raiseamount}{\includegraphics[scale=\resultscale]{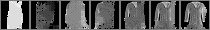}} & \raisebox{\raiseamount}{\includegraphics[scale=\resultscale]{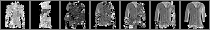}} & \raisebox{\raiseamount}{\includegraphics[scale=\resultscale]{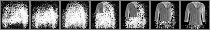}} \\
    & & & \\
    & CNN (complete) & & \\
    truth & \raisebox{\raiseamount}{\includegraphics[scale=\resultscale]{inpainting_growing_fashion_complete_classnet_fashion_quantize16_missing}} & & \\
    observed & \raisebox{\raiseamount}{\includegraphics[scale=\resultscale]{inpainting_growing_fashion_missing_classnet_fashion_quantize16_missing}} & & \\
    prediction & \raisebox{\raiseamount}{\includegraphics[scale=\resultscale]{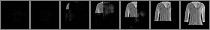}} & & \\
\end{tabular}
\end{table}

\begin{figure}[htbp]
    \centering
    \includegraphics[width=\textwidth]{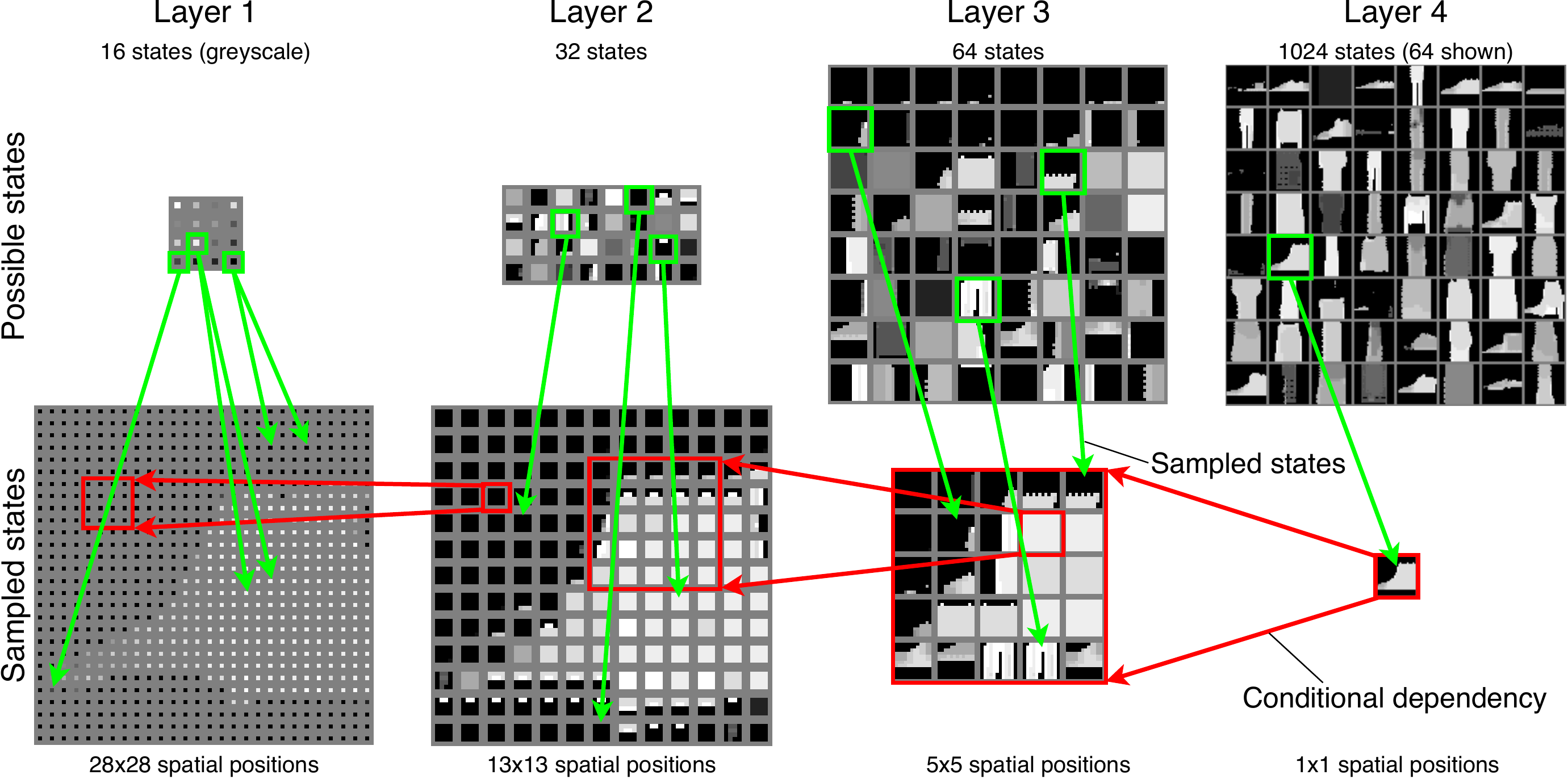}
    \caption{
        \emph{Top row:} learned DLT parts per state per layer, when trained on incomplete Fashion-MNIST.
        \emph{Bottom row:} visualisation of the sampling process.
    }
    \label{fig:parts_per_layer_fashion}
\end{figure}

\end{document}